\documentclass[10pt,journal,compsoc]{IEEEtran}
\ifCLASSOPTIONcompsoc
  \usepackage[nocompress]{cite}
\else
  \usepackage{cite}
\fi
\ifCLASSINFOpdf
\else
\fi
\usepackage{amsmath}
\usepackage{amsthm}
\usepackage{amssymb}
\usepackage{mathtools}

\usepackage{multirow}
\usepackage{booktabs}
\usepackage{nicefrac}
\usepackage{url}

\newtheorem{theorem}{Theorem}

\newtheorem{proposition}{Proposition}

\newcommand{\eg}{e.\,g., }
\newcommand{\ie}{i.\,e., } 

\DeclareMathOperator*{\argmin}{arg\,min}

\begin{document}
\title{Permute Me Softly: Learning Soft Permutations for Graph Representations}

\author{Giannis~Nikolentzos,
        George~Dasoulas,
        and~Michalis~Vazirgiannis% <-this % stops a space
\IEEEcompsocitemizethanks{\IEEEcompsocthanksitem G. Nikolentzos, G. Dasoulas and M. Vazirgiannis are with the Computer Science Laboratory, \'Ecole Polytechnique, Palaiseau 91120, France.\protect\\
E-mail: nikolentzos@lix.polytechnique.fr}% <-this % stops an unwanted space
\thanks{Manuscript received April 19, 2005; revised August 26, 2015.}}


\markboth{}%
{Shell \MakeLowercase{\textit{et al.}}: Bare Demo of IEEEtran.cls for Computer Society Journals}



\IEEEtitleabstractindextext{%
\begin{abstract}
Graph neural networks (GNNs) have recently emerged as a dominant paradigm for machine learning with graphs.
Research on GNNs has mainly focused on the family of message passing neural networks (MPNNs).
Similar to the Weisfeiler-Leman (WL) test of isomorphism, these models follow an iterative neighborhood aggregation procedure to update vertex representations, and they next compute graph representations by aggregating the representations of the vertices.
Although very successful, MPNNs have been studied intensively in the past few years.
Thus, there is a need for novel architectures which will allow research in the field to break away from MPNNs.
In this paper, we propose a new graph neural network model, so-called $\pi$-GNN which learns a ``soft'' permutation (\ie doubly stochastic) matrix for each graph, and thus projects all graphs into a common vector space.
The learned matrices impose a ``soft'' ordering on the vertices of the input graphs, and based on this ordering, the adjacency matrices are mapped into vectors.
These vectors can be fed into fully-connected or convolutional layers to deal with supervised learning tasks.
In case of large graphs, to make the model more efficient in terms of running time and memory, we further relax the doubly stochastic matrices to row stochastic matrices.
We empirically evaluate the model on graph classification and graph regression datasets and show that it achieves performance competitive with state-of-the-art models.
\end{abstract}

\begin{IEEEkeywords}
Graph Neural Networks, Permutation Matrices, Graph Representations.
\end{IEEEkeywords}}

\maketitle

\IEEEdisplaynontitleabstractindextext

% \fi
\IEEEpeerreviewmaketitle

\IEEEraisesectionheading{\section{Introduction}\label{sec:introduction}}

\IEEEPARstart{G}{raphs} arise naturally in many contexts including chemoinformatics, bioinformatics and social networks.
For instance, in chemistry, molecules can be modeled as graphs where vertices and edges represent atoms and chemical bonds, respectively.
A social network is usually represented as a graph where users are mapped to vertices and edges capture friendship relationships.
Due to the recent growth in the amount of produced graph-structured data, the field of graph representation learning has attracted a lot of attention in the past years with applications ranging from drug design \cite{kearnes2016molecular} to learning the dynamics of physical systems \cite{pfaff2020learning}.

Among the different algorithms proposed in the field of graph representation learning, graph neural networks (GNNs) have recently shown significant success in solving real-world learning problems on graphs.
It is probably not exaggerated to say that the large majority of the GNNs proposed in the past years are all variants of a more general framework called Message Passing Neural Networks (MPNNs) \cite{gilmer2017neural}.
Indeed, with the exception of a few architectures \cite{niepert2016learning,maron2018invariant,maron2019provably,nikolentzos2020random,toenshoff2021graph}, all the remaining models belong to the family of MPNNs.
For a certain number of iterations, these models update the representation of each vertex by aggregating information from its neighborhood.
In fact, there is a close connection between this update procedure and the Weisfeiler-Leman test of isomorphism \cite{weisfeiler1968reduction}.
To generate a representation for the entire graph, MPNNs usually apply some permutation invariant readout function to the updated representations of the vertices.
MPNNs have been studied intensively in the past years, thus it comes as no surprise that even their expressive power has been characterized \cite{xu2019how,morris2019weisfeiler,morris2020weisfeiler}.

Current research on graph representation learning seems to be dominated by MPNN models.
Therefore, there is a need to develop novel architectures which will allow research in the field to break away from the message passing schemes which have been extensively studied in the past years.
This paper takes a step towards this direction.
Specifically, we first highlight some general limitations of approaches that project graphs into vector spaces.
We consider a well-established distance measure for graphs, and we show that in the general case, there is no inner product space such that the distance induced by the norm of the space matches exactly the considered distance function.
We also show that, for specific classes of graphs, such representations can be generated by imposing some ordering on the vertices of each graph.
The above result motivates the design of a novel neural network model, so-called $\pi$-GNN which learns a ``soft'' permutation (\ie doubly stochastic) matrix for each graph, and thus projects all graphs into a common vector space.
The learned matrices impose a ``soft'' ordering on the vertices of the graphs, and based on this ordering, the adjacency matrices are mapped into vectors.
These vectors can then be fed into fully-connected or convolutional layers to deal with supervised learning tasks.
To make the model more efficient in terms of running time and memory, we further relax the doubly stochastic matrices to row stochastic matrices.
We compare the performance of the proposed model to well-established neural architectures on several benchmark datasets for graph classification and graph regression.
Results show that the proposed model matches or outperforms competing methods.
Our main contributions are summarized as follows:
\begin{itemize}
    \item  We demonstrate that graph embedding algorithms such as GNNs, cannot isometrically embed a metric space whose metric corresponds to a widely accepted distance function for graphs into a vector space.
    This is, however, possible for specific classes of graphs.
    \item We propose a novel neural network model, $\pi$-GNN, which learns a ``soft'' permutation matrix for each graph and uses this matrix to project the graph into a vector space.
    The set of permutation matrices induces an alignment of the input graphs, and thus graphs are mapped into a common space.
    \item We evaluate the proposed model on several graph classification and graph regression datasets where it achieves performance comparable and in some cases better than that of state-of-the-art GNNs.
\end{itemize}

The rest of this paper is organized as follows.
Section~\ref{sec:related_work} provides an overview of the related work.
Section~\ref{sec:expressiveness} highlights the limitations of graph embedding approaches.
Section~\ref{sec:contribution} provides a detailed description of the proposed $\pi$-GNN model.
Section~\ref{sec:experiments} evaluates the proposed model in graph classification and graph regression tasks.
Finally, Section~\ref{sec:conclusion} concludes.

\section{Related Work}\label{sec:related_work}
The first graph neural network (GNN) models were proposed several years ago \cite{sperduti1997supervised,micheli2009neural,scarselli2009graph}, however, it is not until recently that these models started to attract considerable attention.
The architectures that helped to make the field active include modern variants of the old models \cite{li2016gated} as well as approaches that generalized the convolution operator to graphs based on well-established graph signal processing concepts \cite{bruna2014,defferrard2016convolutional,kipf2016semi}.
Though motivated differently, it was later shown that all these models are instances of a general message passing framework (MPNNs) \cite{gilmer2017neural} which consists of two phases.
First, a message passing phase where vertices iteratively update their feature vectors by aggregating the feature vectors of their neighbors.
Second, a readout phase where a permutation invariant function is employed to produce a feature vector for the entire graph.
MPNNs have been extensively studied in the past years and there is now a large body of research devoted to these models \cite{zhang2018end,xu2019how,murphy2019relational}.
Research has focused mainly on the message passing phase \cite{pgso2021}, but also on the readout phase though in a smaller extent \cite{such2017robust,ying2018hierarchical}.
%Recently, the utilization of parametrized graph shift operators was introduced for more powerful message passing operations \cite{pgso2021}.
The family of MPNNs is closely related to the Weisfeiler-Lehman (WL) isomorphism test \cite{weisfeiler1968reduction}.
Specifically, these models generalize the relabeling procedure of the WL test to the case where vertices are annotated with continuous features.
Standard MPNNs have been shown to be at most as powerful as the WL test in distinguishing non-isomorphic graphs \cite{xu2019how,morris2019weisfeiler}.
Some works aggregate other types of structures instead of neighbors such as small subgraphs \cite{lee2019graph} or paths \cite{chen2020convolutional}.
Considerable efforts have also been devoted to building deeper MPNNs \cite{li2019deepgcns,gallicchio2020fast}, more powerful models \cite{morris2019weisfeiler,morris2020weisfeiler,you2021identity}, and models that can capture the position of a vertex with respect to all other vertices in the graphs \cite{you2019position}.
There have also been made some efforts to develop GNN models that do not follow the design paradigm of MPNNs.
For instance, Niepert \textit{et al.} proposed a model that extracts neighborhood subgraphs for a subset of vertices, imposes an ordering on each subgraph's vertices, and applies a convolutional neural network on the emerging matrices \cite{niepert2016learning}.
Maron \textit{et al.} proposed $k$-order graph networks, a general form of neural networks that consist of permutation equivariant and invariant tensor operations.
Instances of these models correspond to a composition of functions, typically a number of equivariant linear layers and an invariant linear layer which is followed by a multi-layer perceptron \cite{maron2018invariant,maron2019provably}.
Nikolentzos and Vazirgiannis proposed a model that generates graph representations by comparing the input graphs against a number of latent graphs using random walk kernels \cite{nikolentzos2020random}.
%Toenshoff \textit{et al.} proposed CRaWl, a model that uses random walks to extract features which are then processed by a standard convolutional neural network \cite{toenshoff2021graph}.

The works closest to ours are the ones reported in \cite{bai2019learning} and in \cite{bai2020learning}.
In these works, the authors map input graphs into fixed-sized aligned grid structures.
To achieve that, they align the vertices of each graph to a set of prototype representations.
To obtain these prototype representations, the authors apply the $k$-means algorithm to the vertices of all graphs and each emerging centroid is considered as a prototype.
They then compute for each input graph a binary correspondence matrix by assigning each vertex of the graph to its closest centroid.
Based on this matrix, the authors produce a new graph (\ie a new adjacency matrix and a new matrix of features), while they also impose an ordering on the vertices of the new graph using some centrality measure.
Finally, these matrices are fed into an MPNN model which is followed by a convolutional neural network to generate the output.
Unfortunately, these models are not end-to-end trainable.
Mapping input graphs into aligned grid structures involves non-differentiable operations and is thus applied as a preprocessing step.
Furthermore, for large datasets, partitioning the vertices of all graphs into clusters using the $k$-means algorithm can be computationally expensive.
On the other hand, the proposed model is better motivated and is end-to-end trainable since it employs a differentiable layer to compute a matrix of ``soft'' correspondences.
Our work is also related to neural network models that learn to compare graphs to each other \cite{bai2018convolutional,bai2019simgnn,wang2019learning,santa2018visual}.

\section{Can We Generate Expressive Graph Representations?}\label{sec:expressiveness}
As mentioned above, in several application areas, samples do not come in the form of fixed-sized vectors, but in the form of graphs.
For instance, in chemistry, molecules are typically represented as graphs, and in social network analysis, collaboration patterns are also mapped into graph structures.
There exist several approaches that map graphs into vectors, however, in most cases, the emerging vectors could be of low quality, could be difficult to obtain, and they may fail to capture the full complexity of the underlying graph objects.
However, vector representations are very convenient since most well-established learning algorithms can operate on this type of data.
These are developed in inner product spaces or normed spaces, in which the inner product or norm defines the corresponding metric.
Ideally, we would like to project a collection of graphs into a Euclidean space such that some well-established notion of distance between graphs is captured as accurately as possible from the Euclidean norm of that space.

Before introducing the distance function, we present some key notation for graphs.
Let $[n] = \{1,\ldots,n\} \subset \mathbb{N}$ for $n \geq 1$.
Let $G = (V,E)$ be an undirected graph, where $V$ is the vertex set and $E$ is the edge set.
We will denote by $n$ the number of vertices and by $m$ the number of edges.
The adjacency matrix $\mathbf{A} \in \mathbb{R}^{n \times n}$ is a symmetric matrix used to encode edge information in a graph.
%The element of the $i$-th row and $j$-th column is equal to $1$ if there is an edge between $v_i$ and $v_j$ where $v_i, v_j \in V$, and $0$ otherwise.
%For node-attributed graphs, every node in the graph is associated with a feature vector.
%We use $\mathbf{X} \in \mathbb{R}^{n \times d}$ to denote the matrix of node features where $d$ is the dimension of features.
%The $i$-th row of $\mathbf{X}$ contains the feature of vertex $v_i$.
We say that two graphs $G_1$ and $G_2$ are isomorphic to each other, \ie $G_1 \simeq G_2$, if there exists an adjacency preserving bijection $\pi \colon V_1 \rightarrow V_2$, \ie $(u,v)$ is in $E_1$ if and only if $(\pi(u), \pi(v))$ is in $E_2$, call $\pi$ an isomorphism from $G_1$ to $G_2$. 

The distance function that we consider in this paper is the \emph{Frobenius distance} \cite{grohe2018graph}, one of the most well-studied distance measures for graphs.
Let $G_1= (V_1,E_1)$ and $G_2= (V_2,E_2$) be two graphs on $n$ vertices with respective $n \times n$ adjacency matrices $\mathbf{A}_1$ and $\mathbf{A}_2$.
The \emph{Frobenius distance} is a function $d: \mathcal{G} \times \mathcal{G} \rightarrow \mathbb{R}$ where $\mathcal{G}$ is the space of graphs which quantifies the distance of two graphs and can be expressed as the following minimization problem:
\begin{equation}
  d(G_1, G_2) = \min_{\mathbf{P} \in \Pi} ||  \mathbf{A}_1 - \mathbf{P} \, \mathbf{A}_2 \, \mathbf{P}^\top ||_F \, ,
  \label{eq:distance}
\end{equation}
where $\Pi$ denotes the set of $n \times n$ permutation matrices, and $||\cdot||_F$ is the Frobenius matrix norm.
For clarity of presentation we assume $n$ to be fixed (\ie both graphs consist of $n$ vertices).
In order to apply the function to graphs of different cardinalities, one can append zero rows and columns to the adjacency matrix of the smaller graph to make its number of rows and columns equal to $n$.
Therefore, the problem of graph comparison can be reformulated as the problem of minimizing the above function over the set of permutation matrices.
A permutation matrix $\mathbf{P}$ gives rise to a bijection $\pi : V_1 \rightarrow V_2$.
The function defined above seeks for a bijection such that the number of common edges $|\{ (u,v) \in E_1 : \big(\pi(u),\pi(v)\big) \in E_2 \}|$ is maximized.
The above definition is symmetric in $G_1$ and $G_2$.
The two graphs are isomorphic to each other if and only if there exists a permutation matrix $\mathbf{P}$ for which the above function is equal to $0$.
Unfortunately, the above distance function is not computable in polynomial time \cite{grohe2018graph,arvind2012approximate}.
Furthermore, computing the distance remains hard even if both input graphs are trees \cite{grohe2018graph}.
Its high computational cost prevents the above distance function from being used in practical scenarios.
An interesting question that we answer next is how well graph embedding algorithms, \ie approaches that map graphs into vectors, can embed the space of graphs equipped with the above function into a vector space.
It turns out that for an arbitrary collection of graphs, isometrically embedding the aforementioned metric space into a vector space is not feasible.
\begin{theorem}
  Let $(\mathcal{G}, d)$ be a metric space where $\mathcal{G}$ is the space of graphs and $d$ is the distance defined in Equation~\eqref{eq:distance}.
  The above metric space cannot be embedded in any Euclidean space.
  \label{thm:eucl}
\end{theorem}
%\begin{proof}
%  The proof is left to the supplementary material.
%\end{proof}
Unfortunately, most machine learning algorithms that operate on graphs map graphs explicitly or implicitly into vectors (\eg the readout function of MPNNs).
The above result suggests that these learning algorithms cannot generate maximally expressive graph representations, \ie vector representations such that the Euclidean distances between the different graphs are arbitrarily close (or equal) to those produced by the distance function defined in Equation~\eqref{eq:distance}.

The above result holds for the general case (\ie arbitrary collections of graphs).
If we restrict our set to contain instances of only specific classes of graphs, then, it might be possible to map graphs into vectors such that the pairwise Euclidean distances are equal to those that emerge from Equation~\eqref{eq:distance}.
For instance, if our set consists only of paths or of cycles, we can indeed project them into a vector space such that the Euclidean norm induces the distance defined in Equation~\eqref{eq:distance}.
To obtain such representations, the trick is to impose an ordering on the vertices of each graph such that those orderings are consistent across graphs.
For example, in the case of the path graphs, we can impose the following ordering: the first vertex is one of the two terminal vertices.
Let $v$ denote that vertex.
Then, the $i$-th vertex is uniquely defined and corresponds to the vertex $u$ which satisifies $\text{sp}(v, u) = i-1$ where $\text{sp}(\cdot, \cdot)$ denotes the shortest path distance between two vertices.
Let $n$ denote the number of vertices of the longest path in the input set of graphs.
The adjacency matrix of each graph $G_i$ is then expanded by zero-padding such that $\mathbf{A}_i \in \mathbb{R}^{n \times n}$.
For any pair of graphs $G_1, G_2$, we then have:
\begin{equation*}
  \begin{split}
    d(G_1, G_2) = \min_{\mathbf{P} \in \Pi} ||  \mathbf{A}_1 - \mathbf{P} \, \mathbf{A}_2 \, \mathbf{P}^\top ||_F &= ||  \mathbf{A}_1 - \mathbf{I}_n \, \mathbf{A}_2 \, \mathbf{I}_n^\top ||_F \\
    &= ||  \mathbf{A}_1 - \mathbf{A}_2 ||_F \, ,
  \end{split}
\end{equation*}
where $\mathbf{I}_n$ denotes the $n \times n$ identity matrix. 
Clearly, these graphs can be embedded into vectors in some Euclidean space as follows: $\mathbf{v}_i^{\text{adj}} = \text{vec}(\mathbf{A}_i)$ where $\text{vec}$ denotes the vectorization operator which transforms a matrix into a vector by stacking the columns of the matrix one after another, and $\mathbf{v}_i^{\text{adj}} \in \mathbb{R}^{n^2}$.
Then, the distance between two graphs is computed as:
\begin{equation*}
  d(G_1, G_2) = ||  \mathbf{A}_1 - \mathbf{A}_2 ||_F = ||  \mathbf{v}_1^{\text{adj}} - \mathbf{v}_2^{\text{adj}} ||_2 \, ,
\end{equation*}
where $|| \cdot ||_2$ is the standard $\ell_2$ norm of the input vector.
The emerging vectors can thus be though of as the representations of the path graphs in the Euclidean space $\mathbb{R}^{n^2}$.
But still, as shown next, it turns out that we cannot map these graphs into a vector space whose dimension is smaller than the cardinality of the set itself.
\begin{proposition}
  Let $\{ P_2, P_3, \ldots, P_{n+1}\}$ be a collection of $n$ path graphs, where $P_i$ denotes the path graph consisting of $i$ vertices.
  Let also $\mathbf{X} \in \mathbb{R}^{n \times (n+1)^2}$ be a matrix that contains the vector representations of the $n$ graphs obtained as discussed above (\ie the $i$-th row of matrix $\mathbf{X}$ contains the representation of the $i$-th graph).
  Then, the rank of matrix $\mathbf{X}$ is equal to $n$.
\end{proposition}
%\begin{proof}
%  The proof is left to the supplementary material.
%\end{proof}
This is another negative result since it implies that even in cases where we can impose an ordering on the vertices of the graphs which can be used to align the graphs, we cannot project them into a Euclidean space of fixed dimension and retain all pairwise distances.
%For large datasets, we need to sacrifice some efficiency to obtain the most expressive representations.
Note that the proof of Theorem $1$ and Proposition $1$ presented in this section have been moved to the Appendix to improve paper's
readability.

\section{$\pi$-Graph Neural Networks}\label{sec:contribution}
Graph-level machine learning problems are usually associated with finite sets of graphs $\{G_1,\ldots,G_N\} \subset \mathcal{G}$.
%Given such a set of graphs, is it possible to align them such that we can then project them into a Euclidean space?
According to Theorem~\ref{thm:eucl}, we cannot map those graphs to any Euclidean space such that the pairwise Frobenius distances of graphs are preserved.
But, can we find a mapping to the $n^2$-dimensional space that provides the best approximation to the Frobenius distances?
The objective in this case would be to find a permutation matrix $\mathbf{P}_i^*$ for each graph $G_i$ of the dataset where $i \in [N]$ such that the overall distance between graphs is minimized.
This gives rise to the following problem:
\begin{equation}
  \mathbf{P}_1^*, \ldots, \mathbf{P}_N^* = \argmin_{\mathbf{P}_1, \ldots, \mathbf{P}_N \in \Pi} \sum_{i=1}^N \sum_{j=1}^N ||  \mathbf{P}_i \, \mathbf{A}_i \, \mathbf{P}_i^\top - \mathbf{P}_j \, \mathbf{A}_j \, \mathbf{P}_j^\top ||_F \, .
  \label{eq:alignment}
\end{equation}
The above optimization problem imposes an ordering on the vertices of all graphs of the input set and thus each graph $G_i$ can then be embedded into a common Euclidean space as $\mathbf{v}_i^{\text{adj}} = \text{vec}(\mathbf{P}_i^* \, \mathbf{A}_i \, \mathbf{P}_i^{*\top})$.
Unfortunately, solving the above problem is hard.
Note that the distance function defined in Equation~\eqref{eq:distance} is a special case of the above problem when the number of samples is $N=2$.
Furthermore, it is clear that for each pair of graphs $G_i,G_j$ with $i,j \in [N]$, the distance function of Equation~\eqref{eq:distance} is a lower bound to the distances that emerge from the above permutation matrices.
Thus, for any two graphs $G_i, G_j$, we have:
\begin{equation}\label{eq:dist_ineq}
  \begin{split}
    d(G_i, G_j) &= \min_{\mathbf{P} \in \Pi} ||  \mathbf{A}_i - \mathbf{P} \, \mathbf{A}_j \, \mathbf{P}^\top ||_F \\
    &\leq ||  \mathbf{P}_i^* \, \mathbf{A}_i \, \mathbf{P}_i^{*\top} - \mathbf{P}_j^* \, \mathbf{A}_j \, \mathbf{P}_j^{*\top} ||_F \, .
  \end{split}
\end{equation}
Given two graphs $G_i,G_j$, Equation~\eqref{eq:dist_ineq} implies that the permutation matrices $\mathbf{P}_i^*,\mathbf{P}_j^*$ embed the adjacency matrices $\mathbf{A}_i, \mathbf{A}_j$ in a space at least as separable as the one provided from the Frobenius distances.

\subsection{Learning soft permutations}
Inspired by the above alignment problem, in this paper, we propose a neural network model that performs an alignment of the input graphs and embeds them into a Euclidean space.
Specifically, we propose a model that learns a unique permutation matrix for each graph $G_i$ from the collection of input graphs $G_1,\ldots,G_N$.
These permutation matrices enable the model to project graphs into a common vector space.

In fact, the problem of learning permutation matrices $\mathbf{P}_1,\ldots,\mathbf{P}_N \in \Pi$ has a combinatorial nature and is not feasible in practice.
Therefore, in our formulation, we replace the space of permutations by the space of doubly stochastic matrices.
Such approximate relaxations are common in graph matching algorithms \cite{aflalo2015convex,jiang2017graph}.
Let $\mathcal{D}$ denote the set of $n \times n$ doubly stochastic matrices, \ie nonnegative matrices with row and column sums each equal to $1$.
The proposed model associates each graph $G_i$ with a doubly stochastic matrix $\mathbf{D}_i \in \mathcal{D}$ and the vector representation of each graph is now given by $\mathbf{v}_i^{\text{adj}} = \text{vec} \big( \mathbf{D}_i \, \mathbf{A}_i \, \mathbf{D}_i^\top \big)$.
Note that the price to be paid for the above relaxation is that the emerging graph representations are less expressive since two non-isomorphic graphs can be mapped to the same vector.
In other words, there exist pairs of graphs $G_i, G_j$ with $G_i \not \simeq G_j$ and doubly stochastic matrices $\mathbf{D}_i, \mathbf{D}_j \in \mathcal{D}$ such that $\mathbf{D}_i \, \mathbf{A}_i \mathbf{D}_i^\top = \mathbf{D}_j \, \mathbf{A}_j \mathbf{D}_j^\top$.

To compute these doubly stochastic matrices, we capitalize on ideas from optimal transport \cite{peyre2019computational}.
Specifically, we design a neural network that learns matrix $\mathbf{D}_i$ from two sets of feature vectors, one that contains some structural features of the vertices (and potentially their attributes) of graph $G_i$ and one trainable matrix that is randomly initialized.
Let $n$ denote the number of vertices of the largest graph in the input set of graphs.
We first generate a matrix $\mathbf{Q}_i \in \mathbb{R}^{n \times d}$ for each graph $G_i$ of the collection which contains a number of local vertex features that are invariant to vertex renumbering (\eg degree, number of triangles, etc.).
%In the case of attributed graphs, the vectors of attributes of the vertices can be concatenated to the local features.
Note that for a graph $G_i$ consisting of $n_i$ vertices, the last $n - n_i$ rows of matrix $\mathbf{Q}_i$ are initialized to the zero vector.
Let also $\mathbf{W} \in \mathbb{R}^{n \times d}$ denote a matrix of trainable parameters.
Note that we can think of matrix $\mathbf{W}$ as a matrix whose rows correspond to some latent vertices and which contain the features of those vertices.
Then, the rows of matrix $\mathbf{Q}_i$ are compared against those of matrix $\mathbf{W}$ using some differentiable function $f$.
In our experiments, we have defined $f$ as the inner product between the two input vectors followed by the ReLU activation function, thus we have that $\mathbf{S}_i = \text{ReLU}(\mathbf{Q}_i \, \mathbf{W}^\top) \in \mathbb{R}^{n \times n}$ where $\mathbf{S}_i$ is a matrix of scores or similarities between the vertices of graph $G_i$ and the model's trainable parameters. 

Note that before computing matrix $\mathbf{S}_i$, we can first transform the vertex features using a fully-connected layer, \ie $\tilde{\mathbf{Q}}_i = g(\mathbf{Q}_i \, \mathbf{H} + \mathbf{b})$ where $\mathbf{H} \in \mathbb{R}^{d \times \tilde{d}}$ and $\mathbf{b} \in \mathbb{R}^{\tilde{d}}$ is a weight matrix and bias vector, respectively, and $g$ is a non-linear activation function.
We can also use an MPNN architecture to produce new vertex representations, \ie $\tilde{\mathbf{Q}}_i = \text{MPNN}(\mathbf{A}_i, \mathbf{Q}_i)$.
In fact, the expressive power of the proposed model depends on how well those features capture the structural properties of vertices in the graph.
Thus, highly expressive \text{MPNNs} could lead the proposed model into generating more expressive representations.

For a graph $G_i$, matrix $\mathbf{D}_i \in [0,1]^{n \times n}$ can then be obtained by solving the following problem:
% \begin{equation}
%   \max \displaystyle \sum_{j=1}^n \sum_{k=1}^n \mathbf{S}_i^{j,k} \mathbf{D}_i^{j,k} \; \textrm{ s.t. } \; \mathbf{D}_i \, \mathbf{1}_n = \mathbf{1}_n \; \textrm{ and } \; \mathbf{D}_i^\top \, \mathbf{1}_n = \mathbf{1}_n
% \label{eq:opt}
% \end{equation}
\begin{equation}
\begin{split}
  \max \displaystyle \sum_{j=1}^n &\sum_{k=1}^n \mathbf{S}_i^{j,k} \mathbf{D}_i^{j,k} \\
  & \textrm{s.t.} \\
  \mathbf{D}_i \, \mathbf{1}_n = \mathbf{1}_n \; &\textrm{ and } \; \mathbf{D}_i^\top \, \mathbf{1}_n = \mathbf{1}_n \, ,
  %& \mathbf{D}_i \, \mathbf{1}_n = \mathbf{1}_n \\
  %& \mathbf{D}_i^\top \, \mathbf{1}_n = \mathbf{1}_n 
\end{split}
\label{eq:opt}
\end{equation}
where $\mathbf{1}_n$ is an $n$-element vector of ones, and $\mathbf{S}_i^{j,k}$ and $\mathbf{D}_i^{j,k}$ denote the element of the $j$-th row and $k$-th column of matrices $\mathbf{S}_i$ and $\mathbf{D}_i$.
The emerging matrix $\mathbf{D}_i$ is a doubly stochastic matrix, while the above formulation is equivalent to solving a linear assignment problem.
The solution of the above optimization problem corresponds to the optimal transport \cite{peyre2019computational} between two discrete distributions with scores $\mathbf{S}_i$.
Its entropy-regularized formulation naturally results in the desired soft assignment, and can be efficiently solved on GPU with the Sinkhorn algorithm \cite{cuturi2013sinkhorn}.
It is a differentiable version of the Hungarian algorithm \cite{munkres1957algorithms}, classically used for bipartite matching, that consists in iteratively normalizing $\exp(\mathbf{S}_i)$ along rows and columns, similar to row and column softmax.

By solving the above linear assignment problem, we obtain the doubly stochastic matrix $\mathbf{D}_i$ associated with graph $G_i$.
Then, as described above we get $\mathbf{v}_i^{\text{adj}} = \text{vec} \big( \mathbf{D}_i \, \mathbf{A}_i \, \mathbf{D}_i^\top \big)$.
This $n^2$-dimensional vector can be used as features for various machine learning tasks, \eg graph regression or graph classification.
For instance, for a graph classification problem with $|\mathcal{C}|$ classes, the output is computed as:
\begin{equation*}
    \mathbf{p}_i = \text{softmax}(\mathbf{W}^{(c)} \, \mathbf{v}_i^{\text{adj}} + \mathbf{b}^{(c)}) \, ,
\end{equation*}
where $\mathbf{W}^{(c)} \in \mathbb{R}^{|\mathcal{C}| \times n^2}$ is a matrix of trainable parameters and $\mathbf{b}^{(c)} \in \mathbb{R}^{|\mathcal{C}|}$ is the bias term.
We can even create a deeper architecture by adding more fully-connected layers.
Since we have imposed some ``soft'' ordering on the vertices of each graph, we could also treat the matrix $\mathbf{D}_i \, \mathbf{A}_i \, \mathbf{D}_i^\top$ as an image and apply some standard convolution operation where filters of dimension $h \times n$ (with $h < n$) are applied to the representations of $h$ vertices to produce new features.
The filters are applied to each possible sequence of vertices to produce feature maps of dimension $n-h+1$.
These feature maps can be fed to further convolution/pooling layers and finally to a fully-connected layer.

% The filters are applied to each possible sequence of vertices to produce feature maps $\mathbf{c} = [c_1, c_2,\ldots, c_{n-h+1}]$ with $\mathbf{c} \in \mathbb{R}^{n-h+1}$.
% We then apply a max-over-time pooling operation over the feature map and take the maximum value $\hat{c} = \max\{\mathbf{c}\}$ as the feature corresponding to this particular filter.
% We usually apply multiple filters and we concatenate the maximum values that emerge from the different filters.
% These are then passed on to a fully-connected layer.

\subsection{Vertex attributes}
In case of graphs that contain vertex attributes, the graphs' adjacency matrices do not incorporate the vertex information provided by the attributes.
There is thus a need for taking these attributes into account.
Let $\mathbf{X} \in \mathbb{R}^{n \times d}$ denote the matrix of node features where $d$ is the feature dimensionality.
The feature of a given node $v_i$ corresponds to the $i$-th row of $\mathbf{X}$.
To produce a representation of the graph that takes into account both the structure and the vertex attributes, we can first compute a second matrix of scores or similarities $\mathbf{S}_i^{att}$ between the attributes of the vertices of graph $G_i$ and some trainable parameters $\mathbf{W}^{att} \in \mathbb{R}^{n \times d}$, as follows: $\mathbf{S}_i^{att} = \text{ReLU}(\mathbf{X}_i \, \mathbf{W}^{att \top}) \in \mathbb{R}^{n \times n}$. 
We can then solve a problem similar to that of Equation~\eqref{eq:opt} and obtain matrix $\mathbf{D}_i^{att}$.
Then, the ``soft'' permutation matrix can be obtained as follows:
\begin{equation}
    \mathbf{D}_i = \sigma(\alpha) \mathbf{D}_i^{adj} + \big(1 - \sigma(\alpha) \big) \mathbf{D}_i^{att} \, ,
    \label{eq:combination}
\end{equation}
where $\alpha$ is a trainable parameter and $\sigma$ is the sigmoid function.
Thus, each element of $\mathbf{D}_i$ is a convex combination of the elements of the two matrices $\mathbf{D}_i^{adj}$ and $\mathbf{D}_i^{att}$.
The above can be seen as an attention mechanism that allows the model to focus more on the adjacency matrix or on the vertex attributes.
We can then use the learned doubly stochastic matrices to also explicitly map the vertex attributes into a vector space.
Therefore, for some graph $G_i$, we produce a second vector as follows:
\begin{equation*}
  \mathbf{v}_i^{\text{att}} = \text{vec} \big( \mathbf{D}_i \, \mathbf{X}_i \big) \, ,
\end{equation*}
where $\mathbf{v}_i^{\text{att}} \in \mathbb{R}^{nd}$.
If edge attributes are also available, the adjacency matrix $\mathbf{A}_i$ can be represented as a three-dimensional tensor (\ie $\mathbf{A}_i \in \mathbb{R}^{n \times n \times d_e}$ where $d_e$ is the dimension of edge attributes).
We can, then, apply the ``soft'' permutation to this tensor and then map the emerging tensor into a vector.

\subsection{Scaling to large graphs.}
A major limitation of the proposed model is that its complexity depends on the vertex cardinality of the largest graph contained in the input set.
For instance, if there is a single very large graph (with cardinality $n$), and the number of vertices of the remaining graphs is much smaller than $n$, the model will learn $n \times n$ doubly stochastic matrices and all the graphs will be mapped to vectors in $\mathbb{R}^{n^2}$ even though they could have been projected to a lower-dimensional space.
This problem can be addressed by shrinking the adjacency matrix of the largest graph (\eg by removing some vertices and their adjacent edges).
However, this approach seems problematic since it results into loss of information.

To deal with large graphs, we propose to reduce the number of rows of matrix $\mathbf{W}$ from $n$ to some value $p < n$, \ie $\mathbf{W} \in \mathbb{R}^{p \times d}$.
The above will produce a rectangular (and not square) similarity matrix $\mathbf{S}_i \in \mathbb{R}^{n \times p}$ which will then lead to a rectangular row stochastic correspondence matrix that fulfills the following constraints (instead of the ones of Equation~\eqref{eq:opt}):
\begin{equation}
  \mathbf{D}_i \, \mathbf{1}_p = \mathbf{1}_n \quad \text{and} \quad \mathbf{D}_i^\top \, \mathbf{1}_p = (n/p) \mathbf{1}_n\, .
  \label{eq:constraints}
\end{equation}
To obtain matrix $\mathbf{D}_i$, the model applies the Sinkhorn algorithm.
According to \cite{cuturi2013sinkhorn}, for a general square cost matrix $\mathbf{S}$ of dimension $n \times n$, the worst-case complexity of solving an Optimal Transport problem is $\mathcal{O}(n^3\log n)$.
However, the employed Sinkhorn distances algorithm \cite{cuturi2013sinkhorn} exhibits empirically a quadratic $\mathcal{O}(n^2)$ complexity with respect to the dimension $n$ of the cost matrix.
In our case, complexity is even smaller since $\mathbf{S}$ is not a square matrix, but a rectangular matrix.
Except for the application of the Sinkhorn algorithm, $\pi$-GNN also maps the adjacency matrix of $G_i$ into a vector, by performing the following operation $\mathbf{D}_i \, \mathbf{A}_i \mathbf{D}_i^\top$ which requires $\mathcal{O}(n^2 p + p^2 n)$ time, and since $p \leq n$, this corresponds to $\mathcal{O}(n^2 p)$ time.
Moreover, to project the features (if any) of $G_i$ into a vector, the model performs the following matrix multiplication $\mathbf{D}_i \, \mathbf{X}_i$ which takes $\mathcal{O}(pnd)$ time where $d$ is the dimension of the vertex representations.
These operations can be efficiently performed on a GPU.

\subsection{Dustbins}
For some tasks, not all vertices of an input graph need to be taken into account.
For instance, in some cases, just the existence of a single subgraph in a graph could be a good indicator of class membership.
Furthermore, some graphs may consist of fewer than $p$ vertices (\ie the rows of matrix $\mathbf{W}$).
To let the model suppress some vertices of the input graph and/or some latent vertices, we add to each set of vertices a dustbin so that unmatched vertices are explicitly assigned to it.
This technique is common in graph matching and has been employed in other models \cite{sarlin2020superglue}.
We expand matrix $\mathbf{S}_i$ to obtain $\bar{\mathbf{S}}_i \in \mathbb{R}^{(n+1) \times (p+1)}$ by appending a new row and column, the vertex-to-bin, bin-to-vertex and bin-to-bin similarity scores, filled with a single learnable parameter $z \in \mathbb{R}$:
\begin{equation*}
  \bar{\mathbf{S}}_i^{j,p+1} = \bar{\mathbf{S}}_i^{n+1,k} = \bar{\mathbf{S}}_i^{n+1,p+1} = z \qquad \forall j \in [n+1], k \in [p+1]\, .
\end{equation*}

While vertices of the input graph (resp. latent vertices) will be assigned to a single latent vertex (resp. vertex of the input graph) or the dustbin, each dustbin has as many matches as there are vertices in the other set.
We denote as $\mathbf{a} = [\mathbf{1}_n^\top \, p]^\top$ and $\mathbf{b} = [\mathbf{1}_p^\top \, n]^\top$ the number of expected matches for each vertex and dustbin in the two sets of vertices (\ie input vertices and latent vertices).
The expanded matrix $\bar{\mathbf{D}}_i$ now has the constraints:
\begin{equation*}
  \bar{\mathbf{D}}_i \, \mathbf{1}_{p+1} = \mathbf{a} \quad \text{and} \quad \bar{\mathbf{D}}_i^\top \, \mathbf{1}_{n+1} = \mathbf{b} \, .
\end{equation*}
After the linear assignment problem has been solved, we can drop the dustbins and recover $\mathbf{D}_i$, thus we can retain the first $n$ rows and first $k$ columns of matrix $\bar{\mathbf{D}}_i$.

% \paragraph{Softmax Normalization.}
% The sinkhorn normalization fulfills the requirements of rectangular doubly stochastic solutions.
% However, it is inherently inefficient to compute and runs the risk of vanishing gradients $\nicefrac{\partial \mathbf{D}}{\partial \mathbf{S}}$ \cite{zhang2019learning}.
% Here, we propose to relax this constraint by only applying row-wise softmax normalization on $\mathbf{S}$.
% In fact, this has been employed in a previous work \cite{fey2020deep}.

% \paragraph{Computational Complexity.}
% {\color{red} Needs to be determined}

\section{Experimental Evaluation}\label{sec:experiments}
In this section, we evaluate the performance of the proposed $\pi$-GNN model on a synthetic dataset, but also on real-world graph classification and graph regression datasets.
We also evaluate the runtime performance and scalability of the proposed model.

\subsection{Synthetic Dataset}
To empirically verify that the proposed model can learn representations of high quality, we generated a dataset that consists of $191$ small graphs and for each pair of these graphs, we computed the Frobenius distance by solving the problem of Equation~\eqref{eq:distance}.
Since the Frobenius distance function is intractable for large graphs, we generated graphs consisting of at most $9$ vertices.
Furthermore, each graph is connected and contains at least $1$ edge.
The dataset consists of different types of synthetic graphs.
These include simple structures such as cycle graphs, path graphs, grid graphs, complete graphs and star graphs, but also randomly-generated graphs such as Erd{\H{o}}s-R{\'e}nyi graphs, Barab{\'a}si-Albert graphs and Watts-Strogatz graphs.
Table~\ref{tab:statistics_synthetic} shows statistics of the synthetic dataset that we used in our experiments.

There are $\nicefrac{191*192}{2} = 18,336$ pairs of graphs in total (including pairs consisting of a graph and itself).
Based on this dataset, we generated a regression problem where the task is to predict the number of vertices of the input graph.
In other words, the target of a graph $G$ that consists of $n$ vertices is $y=n$.

\begin{table}[t]
  \centering
  \footnotesize
  \caption{Summary of the synthetic dataset that we used in our experiments.}
  \label{tab:statistics_synthetic}
  \def\arraystretch{1.1}
    \begin{tabular}{l|c} 
      \toprule
      \multicolumn{2}{c}{Synthetic Dataset} \\ 
      \midrule
      Max \# vertices & $9$ \\
      Min \# vertices & $2$ \\
      Average \# vertices & $7.29$ \\ \hline
      Max \# edges & $36$ \\
      Min \# edges & $1$ \\
      Average \# edges & $11.34$ \\ \hline
      \# graphs & $191$ \\
      \bottomrule
    \end{tabular}
\end{table}

We annotated the vertices of all graphs with two features (degree and number of triangles in which a vertex participates), we split the dataset into a training and a validation set, and we trained an instance of the proposed $\pi$-GNN model on the dataset.
We stored the model that achieved the lowest validation loss in the disk and we then retrieved it and performed a forward pass to obtain the representations $\mathbf{v}_i^{\text{adj}}$, $i \in \{1,\ldots,191\}$ of all the graphs that are contained in the dataset.
Given these representations, we computed the distance between each pair of graphs $G_i, G_j$ as $d(G_i, G_j) = || \mathbf{v}_i^{\text{adj}} - \mathbf{v}_j^{\text{adj}} ||_2$.
We then compared these distances against the Frobenius distances between all pairs of graphs.
We experimented with $6$ different values for the number of latent vertices, \ie $5,7,9,11,13$ and $15$.
Note that the largest graph in our dataset consists of $9$ vertices.
To assess how well the proposed model approximates the Frobenius distance function, we employed two evaluation metrics: the mean squared error (MSE) and the Pearson correlation coefficient.
Figure~\ref{fig:mse} illustrates the achieved MSE values and correlation coefficients for the different number of latent vertices.
We compared the proposed model against the following two baselines: ($1$) \textit{random}: this method randomly generates a permutation matrix for each graph, and the permutation matrix is applied to the  graph's adjacency matrix.
Then, the distance between two graphs is computed as the Frobenius norm of the difference of the emerging matrices. ($2$) \textit{uniform}: this baseline generates for all graphs a uniform ``soft'' permutation matrix (\ie all the elements are set equal to $\nicefrac{1}{9}$).
Then, once again, the distance between two graphs is computed as the Frobenius norm of the difference of the emerging matrices.
Besides these two baselines, we also compare the proposed model against GCN \cite{kipf2016semi} and GIN \cite{xu2019how}.
We treat the output of the readout function as the vector representation of a graph.
Then, the distance between two graphs is defined as the Euclidean distance between the graphs' representations.
We observe that for a number of latent vertices equal to $7$ and $9$, the representations produced by the proposed model achieve the lowest MSE values (less than $1$) which demonstrates that they can learn high-quality representations that produce meaningful graph distances.
On the other hand, most of the baselines fail to generate graph representations that can yield distances similar to those produced by the Frobenius distance function (MSE greater than $4$ for most of them).
Furthermore, for all considered number of latent vertices, the distances that emerge from the representations learned by the proposed model are very correlated with the ground-truth distances (correlation approximately equal to $0.9$) and more correlated than the distances produced by any other method.

\begin{figure*}[t]
    \centering
    \includegraphics[width=.8\textwidth]{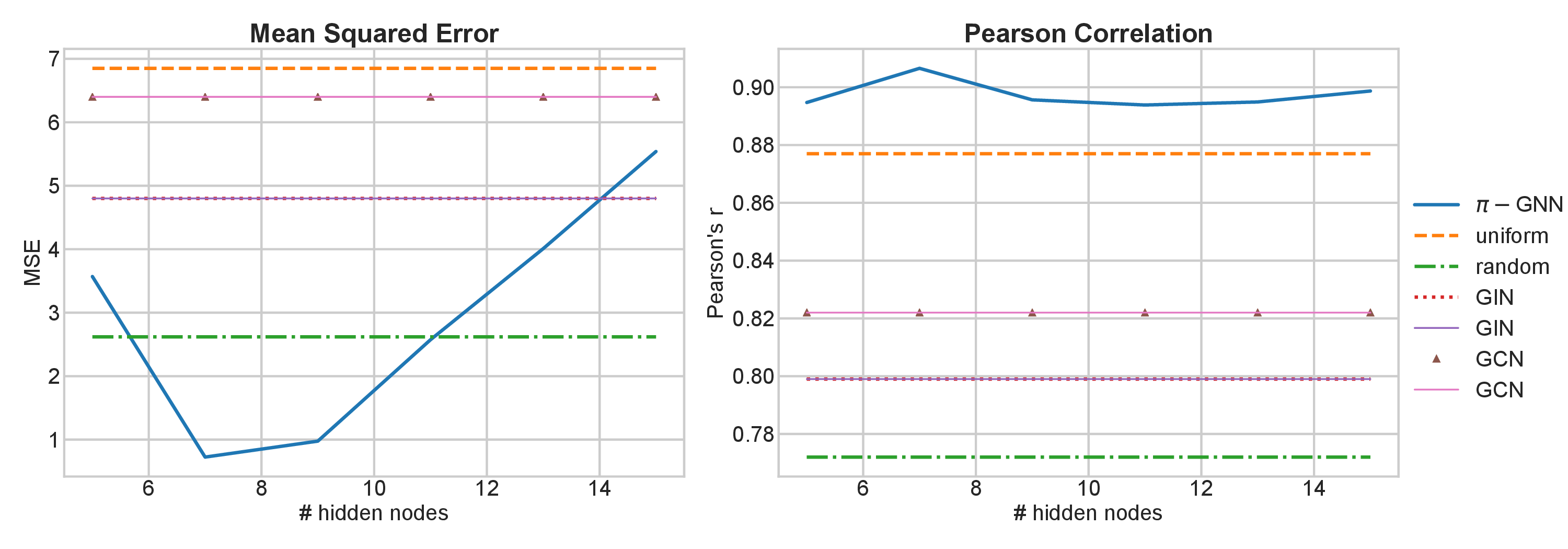}
    \caption{Mean Squared Error and Pearson Correlation of the Frobenius distances with respect to the number of latent nodes}
    \label{fig:mse}
\end{figure*}

In Figure~\ref{fig:heatmap}, we also provide a heatmap that illustrates the $191 \times 191$ matrix of Frobenius distances (left) where the element in the $i$-th row and the $j$-th column corresponds to the Frobenius distance between graphs $G_i$ and $G_j$, and the matrix of distances produced by the proposed model (right).
Clearly, the corresponding values in the two matrices are close to each other, which demonstrates that on datasets that contain small graphs, the proposed model can indeed accurately capture the distance between them.
Due to the prohibitive computational complexity of Frobenius distance, we could not experimentally verify that this also holds in the case of larger graphs.

\begin{figure}[t]
  \centering
  \includegraphics[width=.5\textwidth]{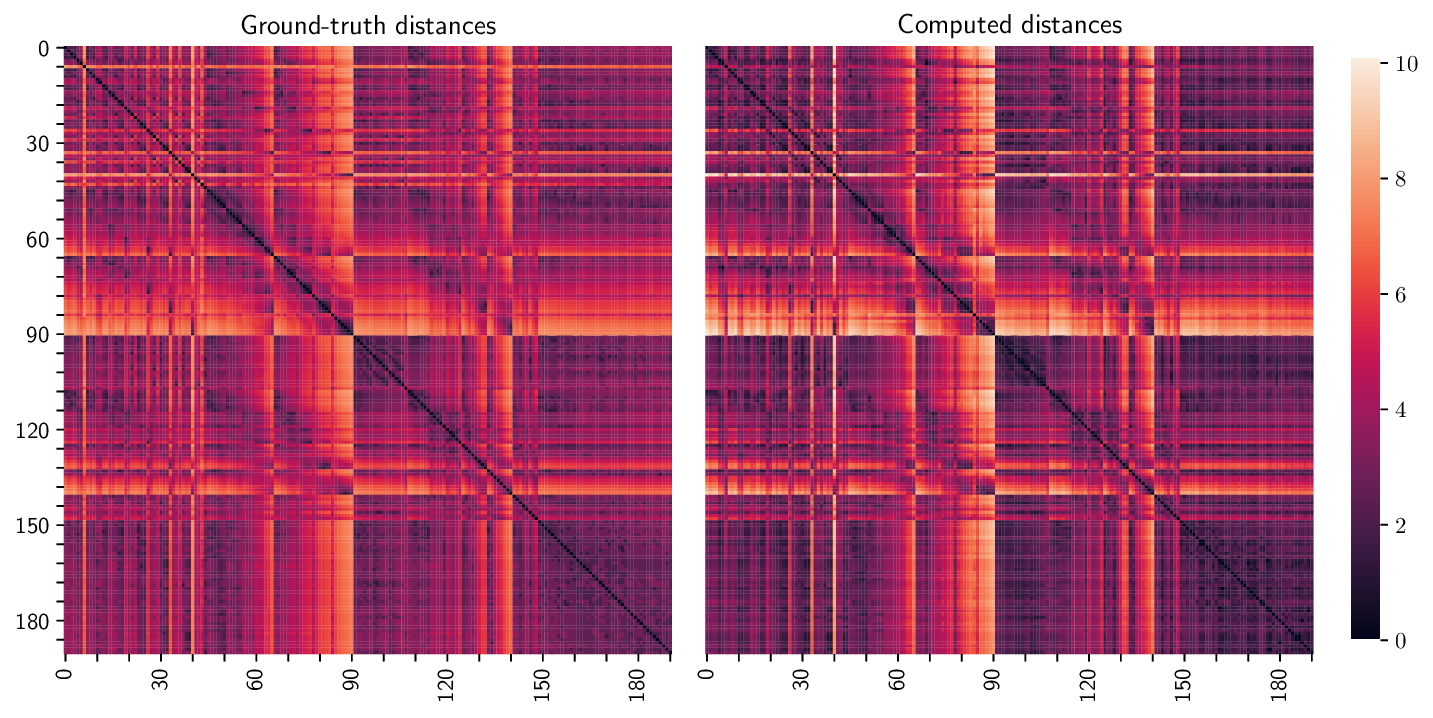}
  \caption{A heatmap of distances produced by the function of Equation~\eqref{eq:distance} and by the proposed model.}
  \label{fig:heatmap}
\end{figure}

\subsection{Real-World Datasets}
\noindent\textbf{Datasets.}
We evaluated the proposed model on the following well-established graph classification benchmark datasets: MUTAG, D\&D, NCI1, PROTEINS, ENZYMES, IMDB-BINARY, IMDB-MULTI, REDDIT-BINARY, REDDIT-MULTI-5K,  and COLLAB \cite{morris2020tudataset}.
A summary of the $10$ datasets is given in Table~\ref{tab:statistics}.
\begin{table*}[t]
  \centering
  \scriptsize
  \caption{Summary of the $10$ datasets that were used in our experiments.}
  \label{tab:statistics}
  \def\arraystretch{1.1}
  \begin{tabular}{lcccccccccc}
  \toprule
  \multirow{2}{*}{\textbf{Dataset}} & \multirow{2}{*}{MUTAG} & \multirow{2}{*}{D\&D} & \multirow{2}{*}{NCI1} & \multirow{2}{*}{PROTEINS} & \multirow{2}{*}{ENZYMES} & IMDB & IMDB & REDDIT & REDDIT & \multirow{2}{*}{COLLAB} \\ 
  & & & & & & BINARY & MULTI & BINARY & MULTI-5K &  \\
  \midrule
  Max \# vertices & 28 & 5,748 & 111 & 620 & 126 & 136 & 89 & 3,782 & 3,648 & 492 \\
  Min \# vertices & 10 & 30 & 3 & 4 & 2 & 12 & 7 & 6 & 22 & 32 \\
  Average \# vertices & 17.93 & 284.32 & 29.87 & 39.05 & 32.63 & 19.77 & 13.00 & 429.61 & 508.50 & 74.49 \\ 
  Max \# edges & 33 & 14,267 & 119 & 1,049 & 149 & 1,249 & 1,467 & 4,071 & 4,783 & 40,119 \\
  Min \# edges & 10 & 63 & 2 & 5 & 1 & 26 & 12 & 4 & 21 & 60 \\
  Average \# edges & 19.79 & 715.66 & 32.30 & 72.81 & 62.14 & 96.53 & 65.93 & 497.75 & 594.87 & 2,457.34 \\ 
  \# labels & 7 & 82 & 37 & 3 & -- & -- & -- & -- & -- & -- \\ 
  \# attributes & -- & -- & -- & -- & 18 & -- & -- & -- & -- & -- \\ 
  \# graphs & 188 & 1,178 & 4,110 & 1,113 & 600 & 1,000 & 1,500 & 2,000 & 4,999 & 5,000 \\ 
  \# classes & 2 & 2 & 2 & 2 & 6 & 2 & 3 & 2 & 5 & 3 \\
  \bottomrule
  \end{tabular}
\end{table*}
MUTAG contains $188$ mutagenic aromatic and heteroaromatic nitro compounds and the task is to predict whether or not each chemical compound has mutagenic effect on the Gram-negative bacterium {\it Salmonella typhimurium} \cite{debnath1991structure}.
D\&D contains more than one thousand protein structures whose nodes correspond to amino acids and edges to amino acids that are less than $6$ \AA ngstroms apart.
The task is to predict if a protein is an enzyme or not \cite{dobson2003distinguishing}.
NCI1 consists of several thousands of chemical compounds screened for activity against non-small cell lung cancer and ovarian cancer cell lines \cite{wale2008comparison}.
PROTEINS contains proteins, and again, the task is to classify proteins into enzymes and non-enzymes \cite{borgwardt2005protein}.
ENZYMES contains $600$ protein tertiary structures represented as graphs obtained from the BRENDA enzyme database, and the task is to assign the enzymes to their classes (Enzyme Commission top level enzyme classes) \cite{borgwardt2005protein}.
IMDB-BINARY and IMDB-MULTI consist of graphs that correspond to movie collaboration networks.
Each graph is the ego-network of an actor/actress, and the task is to predict which genre an ego-network belongs to \cite{yanardag2015deep}.
REDDIT-BINARY and REDDIT-MULTI-5K contain graphs that model interactions between users of Reddit.
Each graph represents an online discussion thread, and the task is to classify graphs into either communities or subreddits \cite{yanardag2015deep}.
COLLAB is a scientific collaboration dataset that consists of the ego-networks of researchers from three subfields of Physics, and the task is to determine the subfield of Physics to which the ego-network of each researcher belongs \cite{yanardag2015deep}.

\begin{table*}[t]
\caption{Classification accuracy ($\pm$ standard deviation) of the proposed model and the baselines on the $10$ benchmark datasets. OOR means Out of Resources, either time ($>$72 hours for a single training) or GPU memory.}%. Best performance per dataset in \textbf{bold}.}
\label{tab:classification_results}
\centering
\footnotesize
\def\arraystretch{1.1}
\begin{tabular}{lccccc}
\toprule
& \textbf{MUTAG} & \textbf{D\&D} & \textbf{NCI1} & \textbf{PROTEINS} & \textbf{ENZYMES} \\
\midrule
DGCNN & 84.0 ($\pm$ 6.7) & 76.6 ($\pm$ 4.3) & 76.4 ($\pm$ 1.7) & 72.9 ($\pm$ 3.5) & 38.9 ($\pm$ 5.7) \\ 
DiffPool & 79.8 ($\pm$ 7.1) & 75.0 ($\pm$ 3.5) & 76.9 ($\pm$ 1.9) & \textbf{73.7} ($\pm$ 3.5) & 59.5 ($\pm$ 5.6) \\ 
ECC & 75.4 ($\pm$ 6.2) & 72.6 ($\pm$ 4.1) & 76.2 ($\pm$ 1.4) & 72.3 ($\pm$ 3.4) & 29.5 ($\pm$ 8.2) \\ 
GIN & 84.7 ($\pm$ 6.7) & 75.3 ($\pm$ 2.9) & \textbf{80.0} ($\pm$ 1.4) & 73.3 ($\pm$ 4.0) & 59.6 ($\pm$ 4.5) \\ 
GraphSAGE & 83.6 ($\pm$ 9.6) & 72.9 ($\pm$ 2.0) & 76.0 ($\pm$ 1.8) & 73.0 ($\pm$ 4.5) & 58.2 ($\pm$ 6.0) \\
\midrule
$\pi$-GNN & \textbf{86.1} ($\pm$ 8.4) & 77.7 ($\pm$ 3.7) & 76.0 ($\pm$ 1.7) & 73.6 ($\pm$ 3.5) & \textbf{60.3} ($\pm$ 4.1) \\
$\pi$-GNN-$d$ & 84.9 ($\pm$ 5.7) & \textbf{78.1} ($\pm$ 3.4) & 76.7 ($\pm$ 1.7) & 72.2 ($\pm$ 3.8) & 56.8 ($\pm$ 6.1) \\
\bottomrule
\end{tabular}
\\
\vspace{.1cm}
\begin{tabular}{lccccc}
\toprule
& \textbf{IMDB-B} & \textbf{IMDB-M} & \textbf{REDDIT-B} & \textbf{REDDIT-5K} & \textbf{COLLAB} \\
\midrule
DGCNN & 69.2 ($\pm$ 3.0) & 45.6 ($\pm$ 3.4) & 87.8 ($\pm$ 2.5) & 49.2 ($\pm$ 1.2) & 71.2 ($\pm$ 1.9) \\ 
DiffPool & 68.4 ($\pm$ 3.3) & 45.6 ($\pm$ 3.4) & 89.1 ($\pm$ 1.6) & 53.8 ($\pm$ 1.4) & 68.9 ($\pm$ 2.0) \\ 
ECC & 67.7 ($\pm$ 2.8) & 43.5 ($\pm$ 3.1) & OOR & OOR & OOR \\ 
GIN & 71.2 ($\pm$ 3.9) & \textbf{48.5} ($\pm$ 3.3) & 89.9 ($\pm$ 1.9) & \textbf{56.1} ($\pm$ 1.7) & 75.6 ($\pm$ 2.3) \\ 
GraphSAGE & 68.8 ($\pm$ 4.5) & 47.6 ($\pm$ 3.5) & 84.3 ($\pm$ 1.9) & 50.0 ($\pm$ 1.3) & 73.9 ($\pm$ 1.7) \\
\midrule
$\pi$-GNN & 70.4 ($\pm$ 3.0) & \textbf{48.5} ($\pm$ 3.5) & \textbf{90.0} ($\pm$ 1.2) & 53.2 ($\pm$ 1.5) & 73.1 ($\pm$ 1.2) \\
$\pi$-GNN-$d$ & \textbf{71.5} ($\pm$ 4.0) & 47.6 ($\pm$ 3.4) & 87.9 ($\pm$ 1.8) & 49.1 ($\pm$ 2.7) & \textbf{75.7} ($\pm$ 1.7) \\
\bottomrule
\end{tabular}
\end{table*}

We also assessed the proposed model's effectiveness on two graph classification datasets from the Open Graph Benchmark (OGB) \cite{hu2020open}, a collection of challenging  large-scale datasets.
Specifically, we experimented with two molecular property prediction datasets: ogbg-molhiv and ogbg-molpcba.
Ogbg-molhiv  is a collection of graphs, that represent molecules \cite{molecules}.
It contains $41,127$ graphs with an average number of $25.5$ nodes per graph and an average number of $27.5$ edges per graph.
Nodes are atoms and the edges correspond to chemical bonds between atoms.
The graphs contain node features, that are processed as in \cite{hu2020open}.
The evaluation metric is ROC-AUC.
Ogbg-molpcba is another molecular property prediction dataset from \cite{hu2020open}.
It contains $437,929$ graphs with an average number of $26$ nodes per graph and anverage number of $28$ edges per graph.
Each graph corresponds to a molecule, where nodes are atoms and edges show the chemical bonds.
The node features are $9$-dimensional and the end task contains $128$ sub-tasks.
The evaluation metric is Average Precision due to the skewness of the class balance.

For the regression task, we conducted an experiment on the QM9 dataset \cite{ramakrishnan2014quantum}.
The QM9 dataset contains approximately $134k$ organic molecules \cite{ramakrishnan2014quantum}.
Each molecule consists of Hydrogen (H), Carbon (C), Oxygen (O), Nitrogen (N), and Flourine (F) atoms and contain up to $9$ heavy (non Hydrogen) atoms. 
Furthermore, each molecule has $12$ target properties to predict.

\begin{table}[t]
    \centering
    \footnotesize
    \def\arraystretch{1.1}
    \caption{Performance on the ogbg-molhiv and ogbg-molpcba datasets.}% For ogbg-molhiv, the evaluation metric is ROC-AUC on the test set and for ogbg-molpcba the metric is the Average Precision (AP) score on the test set.}
    \begin{tabular}{lcc}
    \toprule
              & \textbf{ogbg-molhiv }     & \textbf{ogbg-molpcba}     \\ 
              & ROC-AUC & Avg. Precision \\
    \midrule          
    DGN       & $79.70\pm 0.97$  & $28.85 \pm 0.30$ \\ 
    PNA       & $79.02\pm 1.32$  & $28.38 \pm 0.35$ \\ 
    PHC-GNN   & $79.34\pm 1.16$  & $29.47 \pm 0.26$ \\ 
    GCN       & $76.06 \pm 0.97$ & $20.20 \pm 0.24$ \\ 
    GIN       & $75.58 \pm 1.40$ & $22.66 \pm 0.28$ \\ \hline
    $\pi$-GNN  & $79.12 \pm 1.50$ & $28.11 \pm 0.32$ \\ 
    $\pi$-GNN-$d$ & $79.09 \pm 1.31$ & $28.22 \pm 0.41$ \\
    \bottomrule
    \end{tabular}   
    \label{tab:ogbg_results}
\end{table}

\noindent\textbf{Experimental Setup.}
In the first part of the graph classification experiments, we compare $\pi$-GNN against the following five MPNNs: ($1$) DGCNN \cite{zhang2018end}, ($2$) DiffPool \cite{ying2018hierarchical}, ($3$) ECC \cite{simonovsky2017dynamic}, ($4$) GIN \cite{xu2019how}, and ($5$) GraphSAGE \cite{hamilton2017inductive}.
To evaluate the different methods, we employ the framework proposed in \cite{errica2020fair}.
Therefore, we perform $10$-fold cross-validation and within each fold a model is selected based on a $90\%/10\%$ split of the training set.
Since we use the same splits as in \cite{errica2020fair}, we provide the results reported in that paper for all the common datasets.

In the case of the OGB datasets, we compare $\pi$-GNN against the following models that have achieved top places on the OGB graph classification leaderboard: GCN \cite{kipf2016semi}, GIN \cite{xu2019how}, PNA \cite{corso2020principal}, DGN \cite{beaini2020directional}, and PHC-GNN \cite{le2021parameterized}.
All these baselines belong to the family of MPNNs.
Both datasets are already split into training, validation, and test sets, while all reported results are averaged over $10$ runs.

In the graph regression task, we compare the proposed model against the following five models: ($1$) DTNN \cite{wu2018moleculenet}, ($2$) MPNN \cite{wu2018moleculenet}, ($3$) $1$-$2$-GNN \cite{morris2019weisfeiler}, ($4$) $1$-$2$-$3$-GNN \cite{morris2019weisfeiler}, and ($5$) PPGN \cite{maron2019provably}.
The dataset is randomly split into $80\%$ train, $10\%$ validation and $10\%$ test.
We trained a different network for each quantity. 
For the baselines, we use the results reported in the respective papers.

\begin{table*}[t]
\caption{Mean absolute errors of the proposed model and the baselines on the QM9 dataset.}% Best performance per dataset in \textbf{bold}.}
\label{tab:regression_results}
\centering
\footnotesize
\def\arraystretch{1.1}
\begin{tabular}{lccccccc}
    \toprule
    \multirow{3}{*}{\vspace*{8pt}\textbf{Target}} & \multicolumn{7}{c}{\textbf{Method}} \\
    \cmidrule{2-8}
    & \textbf{DTNN} & \textbf{MPNN} & \textbf{$1$-$2$-GNN} & \textbf{$1$-$2$-$3$-GNN} & \textbf{PPGN} & \textbf{$\pi$-GNN} & \textbf{$\pi$-GNN-$d$} \\
    \midrule
    $\mu$ & 0.244 & 0.358 & 0.493 & 0.476 & \textbf{0.0934} & 0.536 & 0.538 \\
    $\alpha$ & 0.95 & 0.89 & \textbf{0.27} & \textbf{0.27} & 0.318 & 0.374 & 0.372 \\
    $\varepsilon_{\text{HOMO}}$ & 0.00388 & 0.00541 & 0.00331 & 0.00337 & \textbf{0.00174} & 0.00394 & 0.00394 \\
    $\varepsilon_{\text{LUMO}}$ & 0.00512 & 0.00623 & 0.00350 & 0.00351 & \textbf{0.0021} & 0.00419 & 0.00421 \\
    $\Delta\varepsilon$ & 0.0112 & 0.0066 & 0.0047 & 0.0048 & \textbf{0.0029} & 0.0055 & 0.0055 \\
    $\langle R^2 \rangle$ & 17.0 & 28.5 & 21.5 & 22.9 & \textbf{3.78} & 26.34 & 26.16 \\
    \textsc{ZPVE} & 0.00172 & 0.00216 & \textbf{0.00018} & 0.00019 & 0.000399 & 0.000235 & 0.000262 \\
    $U_0$ & 2.43 & 2.05 & 0.0357 & 0.0427 & 0.022 & \textbf{0.0210} & \textbf{0.0210} \\
    $U$ & 2.43 & 2.00 & 0.107 & 0.111 & 0.0504 & 0.0255 & \textbf{0.0244} \\
    $H$ & 2.43 & 2.02 & 0.070 & 0.0419 & 0.0294 & 0.0216 & \textbf{0.0202} \\
    $G$ & 2.43 & 2.02 & 0.140 & 0.0469 & 0.024 & 0.0211 & \textbf{0.0203} \\
    $C_{\text{v}}$ & 0.27 & 0.42 & 0.0989 & \textbf{0.0944} & 0.144 & 0.162 & 0.165 \\
    \bottomrule
  \end{tabular}
\end{table*}

For the proposed $\pi$-GNN, we provide results for two different instances: $\pi$-GNN, and $\pi$-GNN-$d$ that correspond to models without and with dustbins, respectively.
In all our experiments, we annotate each vertex with two structural features: (i) its degree, and (ii) the number of triangles in which it participates.
If vertices are already annotated with attributes, we compute the ``soft'' permutation matrix of Equation~\eqref{eq:combination}.
In case vertices are annotated with discrete labels, we first map these labels to one-hot vectors.

For all standard graph classification datasets, we set the batch size to $64$ and the number of epochs to $300$.
We use the Adam optimizer with a learning rate equal to $10^{-3}$.
Layer normalization \cite{ba2016layer} is applied on the  $\mathbf{v}_i^{\text{adj}}$ and $\mathbf{v}_i^{\text{att}}$ (if any) vector representations of graphs, and the two outputs are fed  to two separate two layer MLPs with hidden-dimension sizes of $256$ and $128$.
The two emerging vectors are then concatenated and further fed to a final two layer MLP. 
The hyperparameters we tune for each dataset and model are the number of latent vertices $\in \{20,30\}$, and the hidden-dimension size of the fully-connected layer we employ to transform the vertex features $\in \{32,64\}$ .
For the experiments on the OGB datasets, we set the batch size to $128$ for ogbg-molhiv and we choose the batch size from $\{128,256\}$ for ogbg-molpcba.
Moreover, we choose the number of latent vertices from $\{20,30,40\}$, while we set the hidden-dimension size of the fully-connected layer that transforms the vertex features to $128$.
All the other experimental settings are same as above.
For the QM9 dataset, we set the batch size to $128$, the number of latent vertices to $40$, the hidden-dimension size of the vertex features to $128$, and we use an adaptive learning rate decay based on validation performance.
All the other experimental settings are same as above. 

The model was implemented with PyTorch \cite{paszke2019pytorch}, and all experiments were run on an NVidia GeForce RTX $2080$ GPU.
The code is available at \url{https://github.com/giannisnik/pi-GNN}.

\noindent\textbf{Graph classification results.}
For the standard graph classification datasets, we report in Table~\ref{tab:classification_results} average prediction accuracies and standard deviations.
We observe that the proposed model achieves the highest classification accuracy on $7$ out of $10$ datasets.
Furthermore, it yields the second best accuracy on $1$ dataset and the third best accuracy on the the remaining $2$ datasets.
On the MUTAG and D\&D datasets, it offers respective absolute improvements of $1.4\%$, and $2.8\%$ in accuracy over the best competitor.
With regards to the rest of the models, GIN also achieves high levels of performance and it outperforms all the other models on $2$ datasets.
On those two datasets (NCI$1$ and REDDIT-$5$K), GIN significantly outperforms the proposed model.
Between the two variants of the proposed model, none of them consistently outperforms the other, however, $\pi$-GNN achieves higher levels of accuracy in most cases.
Interestingly, it appears that the performance of the two variants depends on the type of graphs and the associated task.
For example, $\pi$-GNN significantly outperforms $\pi$-GNN-$d$ on both REDDIT datasets.

Table~\ref{tab:ogbg_results} summarizes the test scores of the proposed model and the baselines on the two OGB graph property prediction datasets.
On both datasets, the two $\pi$-GNN instances achieve high levels of performance, while they also achieve very similar performance to each other.
More specifically, on ogbg-molhiv, the two models outperform $3$ out of the $5$ baselines and achieve a test ROC-AUC score very close to that of the best performing model.
On ogbg-molpcba, they beat $2$ out of the $5$ baselines and again they reach similar levels of performance to that of the model that performed the best.

\noindent\textbf{Graph regression results.}
Table~\ref{tab:regression_results} illustrates mean absolute errors achieved by the different models on the QM9 dataset.
On $4$ out of $12$ targets both variants of the proposed model outperform the baselines providing evidence that $\pi$-GNN can also be very competitive in graph regression tasks.
On most of the remaining targets, the proposed model yields mean absolute errors slightly higher than those of the best performing methods.
However, on two targets ($\mu$ and $\langle R^2 \rangle$), it is largely outperformed by PPGN.
Overall, even though the proposed model utilizes two simple structural features, in most cases, its performance is on par with very expressive models (\eg $1$-$2$-$3$-GNN, PPGN) which are equivalent to higher-dimensional variants of the WL algorithm in terms of distinguishing non-isomorphic graphs.

\begin{table*}[t]
\caption{Classification accuracy ($\pm$ standard deviation) of the proposed model and the baselines on the $10$ benchmark datasets.}%. Best performance per dataset in \textbf{bold}.}
\label{tab:classification_results_gin_gcn}
\centering
\footnotesize
\def\arraystretch{1.1}
\begin{tabular}{lccccc}
\toprule
& \textbf{MUTAG} & \textbf{D\&D} & \textbf{NCI1} & \textbf{PROTEINS} & \textbf{ENZYMES} \\
\midrule
$\pi$-GNN & 86.1 ($\pm$ 8.4) & 77.7 ($\pm$ 3.7) & 76.0 ($\pm$ 1.7) & \textbf{73.6} ($\pm$ 3.5) & \textbf{60.3} ($\pm$ 4.1) \\
$\pi$-GNN-$d$ & 84.9 ($\pm$ 5.7) & 78.1 ($\pm$ 3.4) & 76.7 ($\pm$ 1.7) & 72.2 ($\pm$ 3.8) & 56.8 ($\pm$ 6.1) \\
\midrule
$\pi$-GNN (GCN) & \textbf{86.3} ($\pm$ 8.7) & 75.4 ($\pm$ 2.9) & 76.3 ($\pm$ 1.7) & 73.5 ($\pm$ 3.7) & 56.0 ($\pm$ 5.3) \\
$\pi$-GNN-$d$ (GCN) & 83.9 ($\pm$ 6.8) & \textbf{78.2} ($\pm$ 2.9) & 75.9 ($\pm$ 2.3) & 73.1 ($\pm$ 4.3) & 55.6 ($\pm$ 4.0) \\
\midrule
$\pi$-GNN (GIN) & 85.9 ($\pm$ 8.4) & 76.1 ($\pm$ 2.2) & 80.4 ($\pm$ 1.1) & 72.7 ($\pm$ 3.4) & 52.8 ($\pm$ 4.9) \\
$\pi$-GNN-$d$ (GIN) & 83.8 ($\pm$ 7.3) & 75.2 ($\pm$ 4.1) & \textbf{80.6} ($\pm$ 2.0) & 72.4 ($\pm$ 2.9) & 50.9 ($\pm$ 3.4) \\
\bottomrule
\end{tabular}
\\
\vspace{.1cm}
\begin{tabular}{lccccc}
\toprule
& \textbf{IMDB-B} & \textbf{IMDB-M} & \textbf{REDDIT-B} & \textbf{REDDIT-5K} & \textbf{COLLAB} \\
\midrule
$\pi$-GNN & 70.4 ($\pm$ 3.0) & 48.5 ($\pm$ 357) & 90.0 ($\pm$ 1.2) & 53.2 ($\pm$ 1.5) & 73.1 ($\pm$ 1.2) \\
$\pi$-GNN-$d$ & \textbf{71.5} ($\pm$ 4.0) & 47.6 ($\pm$ 3.4) & 87.9 ($\pm$ 1.8) & 49.1 ($\pm$ 2.7) & \textbf{75.7} ($\pm$ 1.7) \\
\midrule
$\pi$-GNN (GCN) & 70.1 ($\pm$ 4.1) & \textbf{49.2} ($\pm$ 4.3) & 91.0 ($\pm$ 1.4) & 53.8 ($\pm$ 1.2) & 68.6 ($\pm$ 3.2) \\
$\pi$-GNN-$d$ (GCN) & 71.1 ($\pm$ 5.2) & 43.5 ($\pm$ 2.9) & 88.8 ($\pm$ 2.0) & 50.0 ($\pm$ 2.5) & 68.3 ($\pm$ 1.7) \\
\midrule
$\pi$-GNN (GIN) & 69.8 ($\pm$ 3.1) & 49.0 ($\pm$ 3.9) & \textbf{92.0} ($\pm$ 0.9) & \textbf{55.8} ($\pm$ 1.3) & 73.3 ($\pm$ 1.9) \\
$\pi$-GNN-$d$ (GIN) & 70.5 ($\pm$ 4.4) & 45.2 ($\pm$ 3.0) & 90.1 ($\pm$ 1.5) & 53.4 ($\pm$ 2.1) & 72.7 ($\pm$ 2.2) \\
\bottomrule
\end{tabular}
\end{table*}

\subsection{Runtime Analysis}
We next study how the empirical running time of the model varies with respect to the values of $n$ (\ie number of vertices of input graphs) and $p$ (\ie number of latent vertices).
We generated a dataset consisting of $200$ graphs.
All graphs are instances of the Erd{\H{o}}s-R{\'e}nyi graph model and consist of $n$ vertices, while they are divided into two classes (of equal size) where the graphs that belong to the first class are sparser than the ones that belong to the second.
In the first experiment, we set $p=10$ and we increase the number of vertices $n$ of the $200$ input graphs from $10$ vertices to $100$ vertices.
In the second experiment, we set $n=50$, and we increase the number of latent vertices $p$ from $10$ to $100$.
We set the batch size to $64$ and train the model for $100$ epochs.
In both cases, we measure the average running time per epoch.
The results are illustrated in Figure~\ref{fig:complexity}.
\begin{figure}[t]
    \centering
    \includegraphics[width=.5\textwidth]{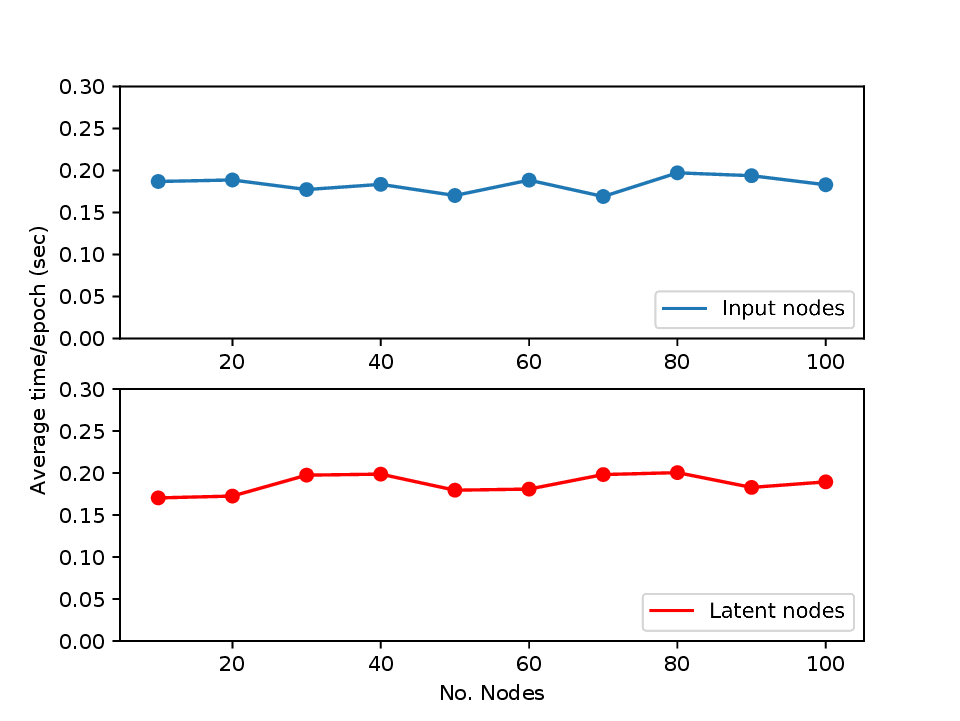}
    \caption{Average running time per epoch with respect to the number of vertices of the input graphs $n$ (top), and to the number of latent vertices $p$ (bottom).}
    \label{fig:complexity}
\end{figure}
The running time of the model seems to be independent of both the size of the input graphs and the number of latent vertices.
We hypothesize that this has to do with the fact that the model mainly performs standard matrix operations that can be efficiently parallelized on a GPU.

\subsection{MPNN-based Node Representations}
In this set of experiments, instead of annotating each node with the two structural features mentioned above, we use MPNNs to produce node representations.
Specifically, we experiment with two MPNN models: ($1$) GCN~\cite{kipf2016semi}; and  ($2$) GIN~\cite{xu2019how}.
We choose the number of message passing layers from $\{2, 3, 4\}$ for both models.
We report in Table~\ref{tab:classification_results_gin_gcn} average prediction accuracies and standard deviations on the $10$ graph classification datasets.
We observe that no model consistently outperforms the other models on the $10$ datasets.
The models that annotate nodes with the two structural features achieve the highest accuracy on $4$ out of the $10$ datasets.
On the other hand, the models that employ GIN to produce node representations outperform the rest of the models on $3$ datasets.
Interestingly, these models outperform the other models by wide margins on these $3$ datasets.
The models that employ GCN to generate node features achieve the highest accuracy also on $3$ datasets.
Overall, we observe that the approach followed to generate node features does not have a significant impact on the performance of the model, thus demonstrating that the model captures different properties of graphs than MPNNs. 

\section{Conclusion}\label{sec:conclusion}
In this paper, we first identified some limitations of graph embedding approaches, and we then proposed $\pi$-GNN, a novel GNN architecture which learns a ``soft'' permutation matrix for each input graph and uses this matrix to map the graph into a vector.
The ability of $\pi$-GNN to approximate the ground-truth graph distances was demonstrated through a synthetic experimental study.
Finally, the proposed model was evaluated on several graph classification and graph regression tasks where it performed on par with state-of-the-art models.

% or


\appendices
\section{Proof of Theorem 1}
% For convenience we restate the Theorem.
% \begin{theorem}
%   Let $(\mathcal{G}, d)$ be a metric space where $\mathcal{G}$ is the space of graphs and $d$ is the distance defined above.
%   The above metric space cannot be embedded in any Euclidean space.
% \end{theorem}

\begin{proof}
The $N \times N$ Euclidean distance matrix $\boldsymbol{\Delta}^2$ contains pairwise distances between all $N$ data points in a dataset.
Suppose a dataset consists of the $5$ graphs shown in Figure~\ref{fig:counter_example}.
Then, we obtain the following matrix of squared distances:
\begin{equation*}
  \boldsymbol{\Delta}^2 = \begin{bmatrix}
  d_{11}^2 & d_{12}^2 & d_{13}^2 & d_{14}^2 & d_{15}^2 \\
  d_{21}^2 & d_{22}^2 & d_{23}^2 & d_{24}^2 & d_{25}^2 \\
  d_{31}^2 & d_{32}^2 & d_{33}^2 & d_{34}^2 & d_{35}^2 \\
  d_{41}^2 & d_{42}^2 & d_{43}^2 & d_{44}^2 & d_{45}^2 \\
  d_{51}^2 & d_{52}^2 & d_{53}^2 & d_{54}^2 & d_{55}^2
  \end{bmatrix}
  = \begin{bmatrix}
  0 & 2 & 6 & 4 & 4 \\
  2 & 0 & 4 & 2 & 6 \\
  6 & 4 & 0 & 6 & 2 \\
  4 & 2 & 6 & 0 & 4 \\
  4 & 6 & 2 & 4 & 0
  \end{bmatrix}
\end{equation*}

The matrix of squared distances $\boldsymbol{\Delta}^2$ can be transformed into a matrix of similarities as follows:
\begin{equation*}
  \mathbf{K} = -\frac{1}{2} \mathbf{J} \, \boldsymbol{\Delta}^2 \, \mathbf{J}
\end{equation*}
where $\mathbf{J}$ is the centering matrix $\mathbf{J} = \mathbf{I} - \frac{1}{5} \mathbf{1} \, \mathbf{1}^\top \in \mathbb{R}^{5 \times 5}$ and $\mathbf{I}$ is the $5 \times 5$ identity matrix.
Matrix $\mathbf{K}$ is not positive definite or positive semidefinite since it has a negative eigenvalue, \ie $\min\{\lambda_1, \ldots, \lambda_5\} = -0.366$.
Therefore, it cannot be decomposed as the following product:
\begin{equation*}
  \mathbf{K} = \mathbf{X} \, \mathbf{X}^\top
\end{equation*}
and it cannot be the Gram matrix of any vector representations of the graphs $\mathbf{x}_1, \ldots, \mathbf{x}_5$.
Hence, the $5$ graphs cannot be embedded in any Euclidean space.
\end{proof}

\begin{figure}[t]
  \centering
  \includegraphics[width=.45\textwidth]{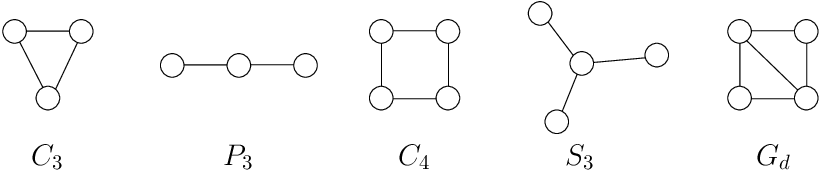}
  \caption{A set of $5$ graph which serve as a counterexample for the proof of Theorem 1.}
  \label{fig:counter_example}
\end{figure}

\section{Proof of Proposition 1}
% For convenience we restate the Proposition.
% \begin{proposition}
%   Let $\{ P_2, P_3, \ldots, P_n\}$ be a set of path graphs, where $P_i$ denotes the path graph consisting of $i$ vertices.
%   Let also $\mathbf{X} \in \mathbb{R}^{(n-1) \times n^2}$ be a matrix that contains the vector representations of the $n-1$ graphs as discussed above.
%   Then, the rank of matrix $\mathbf{X}$ is equal to $n-1$.
% \end{proposition}
\begin{proof}
  Without loss of generality, assume that the path graphs are sorted based on their order (\ie number of vertices).
  Then, the $i$-th row of matrix $\mathbf{X}$ contains the vector representation of graph $P_{i+1}$.
  Notice that given two consecutive path graphs $P_i, P_j$ with $j-i=1$, the first $(i-1)(n+1)+i$ elements of their representations are identical, while the rest of the elements of $P_i$ are equal to zero.
  Therefore, if we subtract the first vector from the second vector, we end up with a vector that has some nonzero elements in the position between $(i-1)(n+1)+i$ and $(j-1)(n+1)+j=i(n+1)+i+1$.
  Then, we can start from the last row of $\mathbf{X}$ and sequentially subtract from each row, the immediately preceding row.
  We obtain a set of $n$ orthogonal vectors (\ie the rows of $\mathbf{X}$) and therefore, the rank of matrix $\mathbf{X}$ is equal to $n$.
\end{proof}

\ifCLASSOPTIONcompsoc
  \section*{Acknowledgments}
\else
  \section*{Acknowledgment}
\fi

The authors would like to thank the NVidia corporation for the donation of a GPU as part of their GPU grant program.

\ifCLASSOPTIONcaptionsoff
  \newpage
\fi



\bibliographystyle{IEEEtran}
\bibliography{biblio}


%\vfill


\end{document}